\documentclass[11pt]{article}
\usepackage[utf8]{inputenc}
\usepackage[T1]{fontenc}
\usepackage{note}

\newcommand{\rob}{\gamma}  % robustness
\newcommand{\cons}{\beta} % consistency
\newcommand{\Dk}{\Delta k}

\usepackage{wrapfig}

\title{Optimal Robustness-Consistency Trade-offs for Learning-Augmented Online Algorithms }

\author{%
 Alexander Wei\\
  UC Berkeley\\
  \texttt{awei@berkeley.edu} \and Fred Zhang \\   UC Berkeley \\
  \texttt{z0@berkeley.edu} 
}
\date{October 21, 2020}

\begin{document}

\maketitle

\begin{abstract}
    We study the problem of improving the performance of online algorithms by incorporating machine-learned predictions. The goal is to design algorithms that are both \textit{consistent} and \textit{robust}, meaning that the algorithm performs well when predictions are accurate \textit{and} maintains worst-case guarantees. Such algorithms have been studied in a recent line of work initiated by Lykouris and Vassilvitskii (ICML '18)  and Kumar, Purohit and Svitkina~(NeurIPS '18). They provide robustness-consistency trade-offs for a variety of online problems. However, they leave open the question of whether these trade-offs are tight, {i.e.}, to what extent to such trade-offs are necessary.
    In this paper, we provide the first set of non-trivial lower bounds  for competitive analysis using machine-learned predictions. We focus on the classic problems of ski rental and non-clairvoyant scheduling and provide optimal trade-offs in various settings.
\end{abstract}
\newpage
\section{Introduction}
The vast gains in predictive ability by machine learning models in recent years have made them an attractive approach for algorithmic problems under uncertainty: One can train a model to predict outcomes on historical data and then respond according to the model's predictions in future scenarios. For example, when renting cloud servers, a company might need to decide whether to pay on-demand or reserve cloud servers for an entire year. The company could try to optimize their purchasing based on a model learned from past demand. However, a central concern for applications like these is the lack of provable bounds on worst-case performance. Modern machine learning models may produce predictions that are embarrassingly inaccurate (e.g., \cite{intriguing}), especially when trying to generalize to unfamiliar inputs. The potential for such non-robust behavior is be problematic in practice, when users of machine learning-based systems desire at least some baseline level of performance in the worst case.

On the other hand, the algorithms literature has long studied algorithms with worst-case guarantees. In particular, the theory of \emph{online algorithms} focuses on algorithms that perform well under uncertainty, even when inputs are chosen adversarially. A key metric in this literature is the \emph{competitive ratio}, which is the ratio between the worst-case performance of an algorithm (without knowledge of the future) and that of an offline optimal algorithm (that has full knowledge of the future).\footnote{For randomized algorithms, one considers the \emph{expected} competitive ratio, which compares the expected cost (taken over the randomness of the algorithm) to the cost of the offline optimal.} That is, an algorithm with a competitive ratio of $\mathcal C$ does at most $\mathcal C$ times worse than any other algorithm, even in hindsight. The classical study of online algorithms, however, focuses on the worst-case outcome over all possible inputs. This approach can be far too pessimistic for many real-world settings, leaving room for improvement in more optimistic scenarios where the algorithm designer has some prior signal about future inputs.

Recent works by Lykouris and Vassilvitskii \cite{DBLP:conf/icml/LykourisV18} and Kumar, Purohit and Svitkina~\cite{NIPS2018_8174} intertwine these two approaches to algorithms under uncertainty by {augmenting}   algorithms with machine learned predictions. The former studies the online caching problem, whereas the latter focuses on the classical   problems of ski-rental and non-clairvoyant job scheduling. They design algorithms that (1) perform excellently when the prediction is accurate and (2) have worst-case guarantees in the form of competitive ratios. For such augmented online algorithms, they introduce the metrics of \emph{consistency}, which measures the competitive ratio in the case where the machine learning prediction is perfectly accurate, and \emph{robustness}, which is the worst-case competitive ratio over all possible inputs. Moreover, the algorithms they design have a natural \emph{consistency-robustness trade-off}, where one can improve consistency at the cost of robustness and {vice versa}. These two works, however, do not discuss the extent to which such trade-offs are necessary, \ie, whether the given trade-offs are tight.\footnote{
We remark that Lykouris and Vassilvitski \cite{DBLP:conf/icml/LykourisV18} prove for the online caching problem that one can achieve constant competitiveness under perfect predictions while having $O(\log k)$ competitiveness in the worst case. This bound is optimal up to constant factors by the classical algorithms for online caching \cite{marker}.}

 In this work, we provide the first set of optimal results for online algorithms using machine-learned predictions. Our results are the following:
 \begin{enumerate}[(i)]
     \item 
 For the ski-rental problem, we give tight lower bounds on the robustness-consistency trade-off in both deterministic and randomized settings, matching the guarantees of the algorithms given by ~\cite{NIPS2018_8174}. 
 \item For the non-clairvoyant job scheduling problem,
 we provide a non-trivial lower bound that is tight at the endpoints of the trade-off curve. 
 
 Moreover, for the case of two jobs, we give matching upper and lower bounds on the full trade-off. The algorithm improves significantly upon that of \cite{NIPS2018_8174}.
 \end{enumerate}
Conceptually, our results show that merely demanding good performance under perfect prediction can require substantial sacrifices in overall robustness. That is, this trade-off between good performance in the ideal setting and overall robustness is deeply intrinsic to the design of learning-augmented online algorithms.

\subsection{Our results}
\begin{wrapfigure}{r}{0.4\textwidth}
  \begin{center}
    \includegraphics[width=0.38\textwidth]{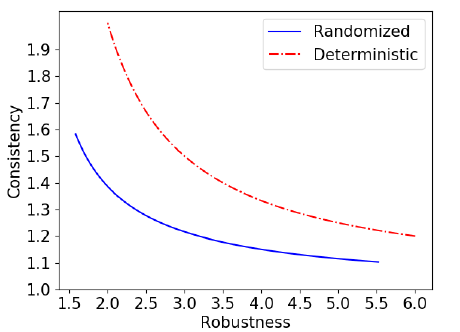}
  \end{center}
  \caption{Tight deterministic and randomized trade-offs for learning-augmented ski-rental.}
\end{wrapfigure}
\paragraph{Ski-rental.}
The \emph{ski rental problem} is a classical online algorithms problem \cite{Karlin1988} with a particularly simple model of decision-making under uncertainty. In the problem, there is a skier who is out to ski for an \emph{unknown} number of days. The first morning, the skier must either rent skis for a cost of \$$1$ or buy skis for a cost of \$$B$. Each day thereafter, the skier must make the same decision again as long as she has not yet purchased skis. The goal for the skier is to follow a procedure that minimizes competitive ratio. Variations of the ski-rental problem have been used to model a diverse set of scenarios, including snoopy caching~\cite{Karlin1988}, dynamic TCP acknowledgement~\cite{karlin2001dynamic}, and renting cloud servers \cite{khanafer2013constrained}.

In our setting of ski-rental with a machine-learned prediction, we assume that, in addition to knowing $B$, the skier has access to a prediction $y$ for the number of days she will ski. Let $\eta$ denote the absolute error of the prediction $y$ (i.e., if she actually skis for $x$ days, then $\eta = |x - y|$). Furthermore, define $c(\eta)$ to be the skier's worst-case competitive ratio over all $y$ given $\eta$. We say that the procedure is \emph{$\rob$-robust} if $c(\eta) \leq \rob$ for any $\eta$ and that it is \emph{$\beta$-consistent} if $c(0) \leq \beta$. We prove deterministic and randomized lower bounds on the robustness-consistency trade-off that match the algorithmic results in~\cite{NIPS2018_8174}. Specifically, we show:

\begin{theorem}[Deterministic Lower Bound for Ski-Rental; also in \cite{gollapudi2019online,recent}]\label{thm:det-lo}
Let $\lambda \in (0,1)$ be a fixed parameter. Any $(1+\lambda)$-consistent deterministic algorithm  for   ski-rental with machine-learned prediction problem is at least $(1+1/\lambda)$-robust.
\end{theorem}
We remark that this deterministic bound is simple to prove and has also appeared in two prior works~\cite{gollapudi2019online,recent}. 

\begin{theorem}[Randomized Lower Bound for Ski-Rental]\label{thm:rand-low}
  Any  (randomized) algorithm  for   ski-rental with machine-learned prediction that achieves robustness $\rob$ must have consistency
  \[ \cons\ge\rob\log\left(1 + \frac{1}{\rob - 1}\right). \]
In particular, any (randomized) algorithm achieving robustness $\rob\le 1/(1 - e^{-\lambda})$ for the ski-rental with machine-learned prediction problem must have consistency $\beta\ge \lambda / (1 - e^{-\lambda})$.
\end{theorem}

\begin{wrapfigure}{r}{0.4\textwidth}
  \begin{center}
    \includegraphics[width=0.38\textwidth]{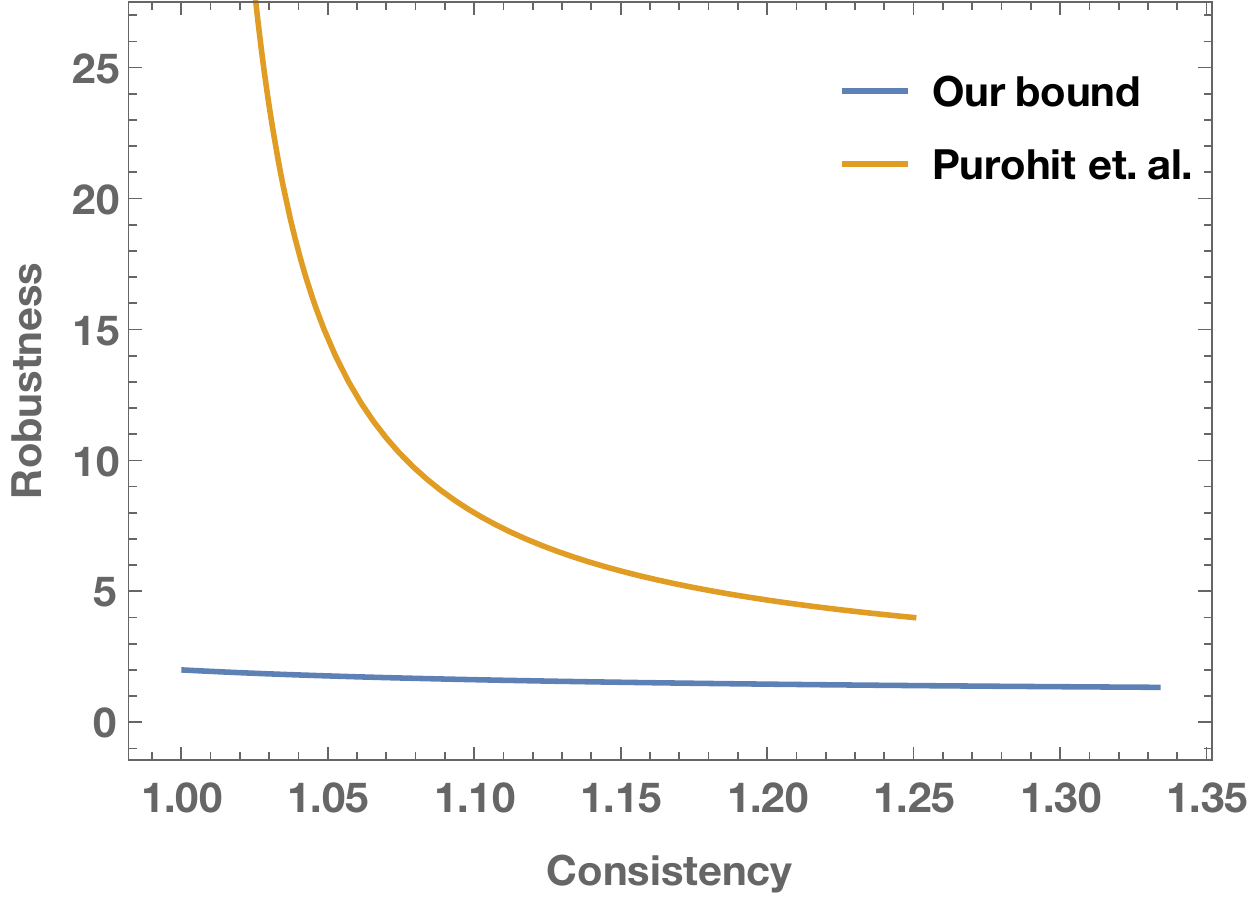}
  \end{center}
  \caption{Tight trade-offs for scheduling two jobs}\label{fig:sc}
\end{wrapfigure}
\paragraph{Non-clairvoyant scheduling.}
The \emph{non-clairvoyant job scheduling problem} was first studied in an online setting by Motwani, Phillips, and Torng \cite{MOTWANI199417}. This problem models scheduling jobs on a single processor, where the jobs have unknown processing times and the objective is to minimize the completion time (\ie, the sum of the job completion times). More formally, the algorithm initially receives $n$ job requests with \textit{unknown} processing times $x_1,x_2,\cdots, x_n$ and is asked to schedule them on a single machine, allowing for preemptions. If the completion time of job $i$ is $t_i$, then the total \emph{completion time} of the algorithm is $\sum_{i=1}^n t_i$.

In the learning-augmented version of the problem, we additionally provide the algorithm with predictions $y_1,y_2,\cdots, y_n$ of the processing times $x_1,x_2, \cdots, x_n$. Let $\eta = \sum_i |x_i - y_i|$ be the $\ell_1$ error of the prediction and $c(\eta)$ be the algorithm's worst-case competitive ratio  given $\eta$. As before, we say an algorithm is  {$\rob$-robust} if $c(\eta) \leq \rob$ for any $\eta$ and  {$\beta$-consistent} if $c(0) \leq \beta$.
Our first result is a lower bound on the robustness-consistency trade-off  in the general case. 
\begin{theorem}[Lower bound for non-clairvoyant scheduling with $n$ jobs]\label{thm:njob}
Any $(1+\lambda)$-consistent algorithm   for  non-clairvoyant scheduling   with machine-learned prediction    must have robustness
\begin{align*}
    \gamma \geq \frac{n + n(n+1)\lambda}{1 + \lambda{(n+1)(n+2)}/{2}}.
\end{align*}
\end{theorem}
This bound is tight at the endpoints of the trade-off. When $\lambda = 0$, we have $c(\eta)\ge n$, which is achieved by any (non-idling) algorithm. On the other hand, when $\lambda = 1 - \smash{\tfrac{2}{n+1}}$ (so $1 + \lambda = 2 - \smash{\tfrac{2}{n+1}}$), we have \smash{$c(\eta)\ge 2 - \tfrac{2}{n+1}$}, which is the tight bound of~\cite{MOTWANI199417} (achieved by round-robin).\footnote{A competitive ratio of $2 - 2 / (n+1)$  can always be achieved (even without ML predictions)~\cite{MOTWANI199417}, so we do not need to consider consistency $1 + \lambda$ for $\lambda\ge 1- 2/(n+1)$}

On the other hand, Kumar, Purohit and Svitkina~\cite{NIPS2018_8174} give an algorithm that is $(1+\lambda)/2\lambda$-consistent and $2/(1-\lambda)$-robust for $\lambda\in (0,1)$.  
In the case of $n=2$, the robustness can be improved to $4/(3-3\lambda)$. We provide a significantly better trade-off  (\cref{fig:sc}) and a matching lower bound in this regime. 
Our algorithm is $2$-competitive over all parameter choices, 
while their algorithm has  robustness tends to infinity as consistency goes to $1$. 
\begin{theorem}[Tight bound for non-clairvoyant scheduling of $2$ jobs]\label{thm:2}
In the case of $2$ jobs, there is an algorithm that achieves $(1+\lambda)$-consistency and $(1+ 1/(1+6\lambda))$-robustness for   non-clairvoyant scheduling with machine-learned prediction,   for any $\lambda\in (0,1/3)$.\footnote{Kumar, Purohit and Svitkina~\cite{NIPS2018_8174}
uses large $\lambda$ to indicate low consistency, whereas we use small $\lambda$ for low consistency. The results are comparable up to a reparametrization. Also,    round-robin has a competitive ratio of $4/3$ for $2$ jobs (without using predictions) \cite{MOTWANI199417}, so we do not need to consider consistency $1 + \lambda$ for $\lambda\ge 1/3$.} Moreover, this bound is tight.
\end{theorem}

\subsection{Related work}
For learning-based ski-rental, the result of \cite{NIPS2018_8174} has since been extended by~\cite{lee2019learning, gollapudi2019online}.   Scheduling with predictions is also studied by \cite{soda20, mitzenmacher2019scheduling, mitzenmacher2019supermarket}, though under different prediction models or problem settings.
The results of \cite{DBLP:conf/icml/LykourisV18} on online caching with ML predictions have been improved and generalized by~\cite{antoniadis2020online,rohatgi2020near,deb20weight,wei2020better}. 
Several other learning-augmented online problems have also been considered in the literature, including  matching, optimal auctions and bin packing~\cite{devanur2009adwords,kumar2018semi,medina2017revenue,antoniadis2020secretary,recent}.

Online algorithms (without ML predictions) are a classical subject in the algorithms literature. The (classic) ski-rental problem is well-understood: It is known that there exists a
$2$-competitive deterministic algorithm~\cite{Karlin1988}. This can be further improved to $e/(e-1)$ using randomization and is known to be optimal~\cite{DBLP:journals/algorithmica/KarlinMMO94}.
There are also numerous extensions of the problem, including snoopy caching~\cite{Karlin1988} and dynamic TCP acknowledgment~\cite{karlin2001dynamic}. 
The non-clairvoyant scheduling problem was first studied by~\cite{MOTWANI199417}. They show that for $n$ jobs the round-robin heuristic achieves a competitive ratio of $2-2/(n+1)$ and  provide a matching lower bound. They also show that randomization provides at most a minor lower-order improvement to the competitive ratio.  Our work revisits these classical results by extending their lower bounds to settings where we want to optimize for consistency (with respect to a prediction) in addition to worst-case competitive ratio.

Another related line of inquiry is the study of online problems in   stochastic settings, where the inputs come from certain distribution~\cite{hentenryck2009online,feldman2009online, DBLP:journals/talg/MahdianNS12, DBLP:conf/soda/MirrokniGZ12, mitzenmacher2019scheduling, esfandiari2018allocation}. We note that this model differs from ours in that we do not make any assumptions on the distribution or stochasticity of inputs.

Finally, using machine learning to design algorithms under uncertainty has been explored in other settings as well, such as  online learning~\cite{kong2018new, bhaskara2020online} and data streams \cite{ali19, Aamand2019LearnedFE, jiang2020learningaugmented, cohen2020composable}. 
A number of works also study learning-based methods for numerical linear algebra, combinatorial optimization, and integer programming~\cite{bello2016neural,khalil2017learning,pmlr-v80-balcan18a,nazari2018reinforcement,NIPS2018_7335,kool19,selsam2018learning,amizadeh2018learning, chawla2019learning,indyk2019learning,alabi2019learning,dao2019learning}.

\subsection{Preliminaries and notations}
In our analyses that follow,  we  use $\mathsf{ALG}$ to denote the cost incurred by the algorithm  on a given input and prediction. We   use $\mathsf{OPT}$ to denote the optimal cost achievable by an algorithm with full knowledge of the future (i.e., an offline algorithm). Note that $\mathsf{ALG}$ is a function of the input and the prediction, while $\mathsf{OPT}$  depends only on the input. The competitive ratio for a given input and prediction is simply the ratio $\mathsf{ALG} / \mathsf{OPT}$.

In terms of this notation, an algorithm is $\cons$-consistent if $\mathsf{ALG} / \mathsf{OPT}\le\cons$ for all situations where the input is the same as the prediction; an algorithm is $\rob$-robust if $\mathsf{ALG} / \mathsf{OPT}\le\rob$ for all pairs of input and prediction.

\section{Ski Rental}
In the ski-rental problem, a skier is out on a ski trip, which will end after the $x$-th day for some \textit{unknown} $x$. Each day, if she has not yet bought skis, the skier may either rent skis for \$$1$ or buy skis for \$$B$ and ski for free from then on. In the learning-augmented version of the problem, the skier is also provided with a machine-learned prediction $y$ of $x$ that she may use to aid her decision. 

We first state the algorithmic results of Kumar, Purohit and Svitkina~\cite{NIPS2018_8174}, which we will prove to be optimal. Their algorithms require a hyperparameter $\lambda \in (0,1)$ that dictates the trade-off between robustness and consistency. Given $\lambda$, the deterministic and randomized algorithms of \cite{NIPS2018_8174} for ski rental with machine-learned predictions proceed as follows:
\begin{algo}
  \textul{$\texttt{Deterministic-Ski-Rental}$}$(y, B)$:\+\\
    If $y \geq B$, \+\\
    Buy at the start of day $\lceil \lambda B \rceil$.\-\\
    Otherwise,\+\\
    Buy at the start of day $\lceil B / \lambda \rceil$.
\end{algo}
\begin{algo}
  \textul{$\texttt{Randomized-Ski-Rental}$}$(y, B)$:\+\\
    If $y \geq B$, \+\\
    Let $k =\lceil \lambda B \rceil$.\-\\
    Otherwise,\+\\
    Let $k= \lceil B / \lambda \rceil$.\-\\
    Select day $i \in [k]$ with probability proportional to $\left(1-1/B\right)^{k-i}$.\\
    Buy at the start of day $i$.
\end{algo}

Kumar, Purohit and Svitkina~\cite{NIPS2018_8174} show that the algorithms achieve the following robustness-consistency trade-offs:

\begin{theorem}[Theorem 2.2 of~\cite{NIPS2018_8174}]\label{thm:ski-det}
Given a parameter $\lambda\in (0,1)$, the \texttt{Deterministic-Ski-Rental} algorithm is $(1 + 1/\lambda)$-robust and $(1 +\lambda)$-consistent. 
\end{theorem}

\begin{theorem}[Theorem 2.3 of~\cite{NIPS2018_8174}]\label{thm:ski-rand}
Given a parameter $\lambda\in (0,1)$, the $\texttt{Randomized-Ski-Rental}$ algorithm is $\bigl(\frac1{1-e^{-(\lambda-1/B)}}\bigr)$-robust and $\bigl(\frac\lambda{1-e^{-\lambda}}\bigr)$-consistent. 
\end{theorem}

Notice that for large $B$, our randomized lower bound  (\autoref{thm:rand-low}) essentially matches the guarantee of~\autoref{thm:ski-rand}.

\subsection{Deterministic Lower Bound}
In this section, we prove~\autoref{thm:det-lo}, which also appeared in~\cite{gollapudi2019online,recent}. Since the algorithm is deterministic, we proceed by a an adversarial argument. 
Let $x$ be the last day of the ski season. The high-level idea is to fix a specific $y$, and then consider two instances, one where $x = y$ and one where $x\neq y$. 
Since the algorithm does not know $x$, it cannot distinguish between these two cases and therefore must output a unique day $t$ (for purchasing skis) given $y$ and $B$. 
Suppose $y$ is large, say, greater than $B$. Then, intuitively, $t$ must be fairly small to satisfy consistency.
Given this constraint, in the other instance, we let the adversary choose an $x\neq y$ that yields the worst possible competitive ratio.
We will show that this competitive ratio indeed matches the robustness upper bound. 
\begin{proof}[Proof of \autoref{thm:det-lo}]
  Let $y$ be the prediction and $\eta = |y-x|$ be the error. Consider a deterministic algorithm that achieves $(1+\lambda)$ consistency.
  Suppose $y>  (1+\lambda) B$, and let $t$ be the day on which the algorithm purchases skis (given $y$ and $B$).  
  
  First, suppose $t \geq  y$.
  When $x= y$, we have $\OPT = B$ and $\ALG =y$. Then the competitive ratio is  $y/B$, which must be bounded by $1+\lambda$ by our consistency requirement, but this contradicts the assumption $y> (1+\lambda)B$.
  Second, suppose $B<t < y$. Again, when $x=y$, $\OPT = B$, and $\ALG = t+B-1$.  By the $(1+\lambda)$-consistency, $(t+B-1)/B \leq 1+\lambda$. Thus, $(t-1)/ B\leq \lambda <1$, contradicting the assumption that $t> B$. 
  Therefore, simply to achieve $(1+\lambda)$-consistency, the algorithm must output $t\le B$. Now under this condition, we consider two cases. We use the case when $y=x$  to derive a bound on $\lambda$, and apply this along with an adversarial argument in the case when $y\neq x$ to obtain our robustness lower bound. 
  \begin{enumerate}[(i)]
      \item Suppose $x = y$. Since $y > B$, we have $\OPT = B$. On the other hand, $\ALG = t+B-1$, as $ t < x$. Thus, the algorithm does $1+ (t-1)/B$ times worse than optimal. Assuming that the algorithm is $(1+\lambda)$-consistent, we have $1+ (t-1)/B \leq 1+\lambda$, so $t\le\lambda B + 1$.
      \item Suppose $x\neq y$. We adversarially set $x = t$; note that $x \leq B$. Thus, $\OPT = x = t$ and $\ALG = t + B-1$. 
      Our bound on $t$ from (i) now lower bounds the competitive ratio as $(t + B - 1) / t\ge 1 + (B - 1) / (\lambda B + 1)$. For large $B$, this lower bound approaches $1 + 1 / \lambda$. 
      This shows that $c(\eta) \geq 1 + 1/\lambda$ and thus completes the proof.\qedhere
  \end{enumerate}
\end{proof}

\subsection{Randomized Lower Bound}
 The starting point of our randomized lower bound is the well-known fact that the ski-rental problem can be expressed as a linear program (see, \eg, \cite{buchbinder2009design}).  Our key observation then is that the consistency and robustness constraints are in fact also linear. Somewhat surprisingly, we show that the resulting linear program can be solved \textit{analytically} in certain regimes.  By exploiting the structure of the linear program, we will determine the optimal robustness for any fixed consistency, and this matches the trade-off given by~\autoref{thm:ski-rand} (when $y \gg B$ and for large $B$).

The proof of our randomized lower bound (\Cref{thm:rand-low}) is fairly technical. Thus, we defer the proof to Appendix~\ref{apx:A} and  only present a   sketch here.

\begin{proof}[Proof sketch of \autoref{thm:rand-low}]
As a first step,
we can characterize algorithms for ski rental as feasible solutions to an infinite linear program, with variables $\{p_i\}_{i\in\mathbb N}$ indicating the probability of buying at day $i$. The constraints of robustness and consistency can be written as linear constraints on this representation. Given $\gamma$ and $\beta$, understanding whether a $\rob$-robust and $\cons$-consistent algorithm exists therefore reduces to checking if this linear program is feasible. (In particular, we do not have an objective for the linear program.)  

First, we ask that the $p_i$'s define a probability distribution. That is, $p_i \geq 0$ and 
\begin{align}
    \sum_{i=1}^\infty p_i= 1.   
\end{align}
Second, to satisfy the consistency constraint, the algorithm must have expected cost within $\beta\cdot \OPT$ when $y = x$. In this case, the ski season ends at $i=y$, so there is no additional cost afterwards.
\begin{align}
   \sum_{i=1}^y (B+i-1)p_i
   +
   y\sum_{i=y+1}^\infty  p_{i}  \leq \beta \min\{ B , y\}.  
\end{align}
Third, each value of $x$ gives a distinct constraint for robustness, where the left side is the expected cost and the right side is $\gamma \cdot \OPT$.  When $x\leq B$, $\OPT = x$, so we have
\begin{align}
\sum_{i=1}^x (B+i-1)p_i + x\sum_{i=x+1}^\infty p_i \leq \rob x
\quad\forall x \leq B.
\end{align}
If $x> B$, then $\OPT = B$. The robustness constraints are  infinitely many, given by
\begin{align}
\sum_{i=1}^x (B+i-1)p_i + x\sum_{i=x+1}^\infty p_i \leq \rob B
\quad\forall x > B.
\end{align}
Having  set up this LP, the remainder of the proof follows in two steps. First, we show that this (infinite) LP can be reduced to a finite one with $B+1$ constraints and $y$ variables. We then proceed to analytically understand  the   solution to the LP. This allows us to lower bound the parameter $\rob$ given any $\cons$, and it indeed matches  the upper bound given by~\cite{NIPS2018_8174}.
\end{proof}

\section{Non-clairvoyant Scheduling}
In the non-clairvoyant scheduling problem, we have to complete $n$ jobs of unknown lengths $x_1, x_2,\cdots, x_n$ using a single processor. The processor only learns the length of a job upon finishing that job.  
The goal in this problem is to schedule the jobs with preemptions to minimize the total completion time, \ie, the sum of the times at which each job finishes. 
Observe that no algorithm can achieve a non-trivial guarantee if preemptions are disallowed. 
The problem has been well-studied in the classic setting. Motwani, Phillips, and Torng  \cite{MOTWANI199417} show that the \textit{round-robin} (RR) algorithm achieves $2-2/(n+1)$ competitive ratio, which is the best possible among deterministic algorithms. 
The algorithm simply assigns a processing rate of $1/k$ to each of the $k$ unfinished jobs at any time. (Note that since preemption is allowed, we can ease our exposition by allowing concurrent jobs run on the processor, with rates summing to at most $1$.)

Now, suppose one has access to a machine-learned oracle that produces predictions $y_1,y_2,\cdots, y_n$ of the processing times $x_1,x_2,\cdots, x_n$. 
Define $\eta = \sum_i |x_i -y_i|$ to be the total prediction error. 
We would like to design algorithms that achieve  a better competitive ratio than $2-2/(n+1)$ when $\eta = 0$ and while preserving some constant worse-case guarantee.

\subsection{A General Lower Bound}
Our first result is a lower bound on the robustness-consistency trade-off that is tight at the endpoints of the trade-off curve. Note that since the classic work~\cite{MOTWANI199417} provides a $c=2-2/(n+1)$ competitive ratio (with no ML prediction), one can always achieve $c$-robustness and $c$-consistency simultaneously. Hence, as we remarked, \autoref{thm:njob} is tight at $\lambda = 0$ and $\lambda =  1 - \tfrac{2}{n+1}$. We now prove the theorem.

\begin{proof}[Proof of \autoref{thm:njob}]
Consider an algorithm that achieves $1+\lambda$ consistency. Let the   predictions be $y_1 = y_2 = \cdots = y_n = 1$. Let $d(i,j)$ denote the amount of processing time on job $i$ before job $j$ finishes. Assume without loss of generality that job $1$ is the first job to finish and that when it finishes, we have $d(i, i)\ge d(j, j)$ for all $i < j$. Consistency requires
\[ (1 + \lambda)\cdot\OPT = \frac{n(n+1)}{2}(1 + \lambda)\ge\sum_{i,j} d(i,j) + \sum_{i=2}^n (n - i + 1)(1 - d(i, i)), \]
where the first term represents the costs incurred thus far, and the second term represents the minimum cost required to finish from this state. Simplifying, we obtain the condition
\begin{equation}\label{eqn}
\frac{n(n+1)}{2}\lambda\ge \sum_{i=2}^n (i - 1)\cdot d(i, i),
\end{equation}
as $d(i, j) = d(i, i)$ for all $i$ at this point in the execution.

Now, consider a (adversarial) setup with $x_i = d(i, i) + \eps$, where we   take $d(i, i)$ to be as measured upon the completion of job $1$ and $\eps > 0$ to be a small positive number. For this instance, we have
\[ \OPT = 1 + \sum_{i=2}^n ix_i + O(\eps). \]
We also have, based on the   execution of the algorithm up to the completion of job $1$, that
\[ \ALG\ge n\bigp{1 + \sum_{i=2}^n x_i}. \]
To show a consistency-robustness lower bound, it suffices to lower bound $\ALG / \OPT$ subject to the consistency constraint. Equivalently, we can upper bound
\[ \frac{\OPT}{\ALG} - \frac 1n\le\frac 1n\bigp{\frac{1 + \sum_{i=2}^n (i-1)x_i + O(\eps)}{1 + \sum_{i=2}^n x_i}}. \]
Suppose we know a priori that the value of the numerator is $C + 1 + O(\eps)$ (i.e., $\sum_{i=2}^n (i-1)x_i = C$). To maximize the quantity on the right-hand side, we would want to have $\sum_{i=2}^n x_i$ be as small as possible subject to the constraints that $x_i\ge x_j\ge 0$ if $i < j$ and
\[ \sum_{i=2}^n (i-1)x_i = C. \]
Observe that this optimization problem is a linear program. For this linear program, suppose we have a feasible solution with $x_i > x_{i+1}$. Such a solution cannot be optimal, as we can set $x_i\gets x_i - \frac{\alpha}{i-1}$ and $x_{i+1}\gets x_{i+1} + \frac{\alpha}{i}$ for sufficiently small $\alpha > 0$, reducing the objective while remaining feasible. 
Thus, if an optimal solution exists, it must have $x_2 =x_3= \cdots = x_n$. It is not hard to see that this linear program is bounded and feasible, so an optimum does exist. It follows that for a given $C$, we want to set $x_2=x_3 = \cdots = x_n = \frac{2C}{n(n-1)}$,
in which case the right-hand side is equal to
\[ \frac{C + 1 + O(\eps)}{1 + \frac{2C}{n}} - \frac{n-1}{2} + \frac{n-1}{2} = \frac{-\frac{n-3}{2} + O(\eps)}{ 1 + \frac{2C}{n}} + \frac{n-1}{2}. \]
To maximize the leftmost term, which has a negative numerator (for sufficiently small $\eps$), we want to maximize $C$. However, we know from \eqref{eqn} that $C = \sum_{i=2}^n (i-1)x_i\le\frac{n(n+1)}{2}\lambda$. Therefore, we have the upper bound
\[ \frac{\OPT}{\ALG} - \frac 1n\le \frac 1n\bigp{\frac{\frac{n(n+1)}{2}\lambda + 1 + O(\eps)}{1 + {(n+1)}\lambda}}. \]
Finally,  taking $\eps\to 0$ yields the desired bound
\[ \frac{\ALG}{\OPT}\ge\frac{n + n(n+1)\lambda}{1 + \frac{(n+1)(n+2)}{2}\lambda}. \tag*{\qedhere} \]
\end{proof}

\subsection{A  Tight Complete Trade-off for Two Jobs}

We now consider   the special case of having $n=2$ jobs. It is always possible to achieve $4/3$ competitiveness by round-robin~\cite{MOTWANI199417}, and with machine-learned predictions,   Kumar, Purohit, and Svitkina \cite{NIPS2018_8174}  proves an  $(1+\lambda)/2\lambda$-consistency and $4/(3-3\lambda)$-robustness trade-off. We show that this trade-off can be significantly improved and that our new bound is in fact tight.

\paragraph{Lower bound.} We start by proving our lower bound. 
Here, we remark that any lower bound for $k$ jobs directly implies the same lower bound for any $n\geq k$ jobs, since one can add $n-k$ dummy jobs with $0$ predicted and actual processing times. 
Thus, the lemma below also holds for $n > 2$. 
\begin{lemma}[Lower bound for non-clairvoyant scheduling]\label{lem:lbs}
For the non-clairvoyant scheduling problem of $2$ jobs, any algorithm that    achieves $(1+\lambda)$-consistency must be at least $1+(1/(1+6\lambda))$-robust for a $\lambda\in (0,1/3)$. 
\end{lemma}
\begin{proof}
Consider a $(1+\lambda)$-consistent algorithm $\mathcal{A}$.  
Suppose the inputs are predictions $y_1 = y_2 =1$.  
First, we focus on an instance $I$, where  $x_1 = y_1,x_2=y_2$. 
Let $d(i,j)$ denote the amount of processing time on job $i$ before job $j$ finishes for this instance,
and assume without loss of generality  that the algorithm finishes job $1$ first. 
Observe   in this scenario the consistency requirement asks that $\mathcal{A}$ must produce a schedule with total completion time at most $(1+\lambda) (2y_1 + y_2)=3+3\lambda$. As job $1$ finishes first, $d(1,2)=1$. Since $x_1=x_2=1$ and $\textsf{ALG} = x_1+x_2 + d(1,2)+d(2,1)$, we must have 
\begin{equation}\label{eq:scl}
    d(2,1) \leq \lambda(2y_1 + y_2) = 3\lambda.
\end{equation}
Now we consider an adversarial instance $I'$ with same predictions ($y_1=y_2=1$), but different choices of actual processing times. In particular, let $x_1 = 1$ but $x_2= d(2,1)+\epsilon$ for an infinitesimal constant $\epsilon$.
Since the inputs to the algorithm are the same as in the previous instance $I$, it would start off by producing the same schedule. 
In particular, the algorithm would finish job $1$ first at time $1+ d(2,1)$, then finish job $2$ immediately afterwards. 
Therefore,
\begin{equation}
    \textsf{ALG} = 2+2d(2,1)+\epsilon.
\end{equation}
On the other hand,
since $\lambda \leq 1/3$, $x_2 \leq x_1$, we have
\begin{equation}
    \textsf{OPT}= 2x_2 + x_1 = 2d(2,1) + 2\epsilon + 1.
\end{equation}
By~\eqref{eq:scl}, we get that the competitive ratio is at least $1+ 1/(1+6\lambda)$ as $\epsilon \rightarrow 0$.
\end{proof}

\paragraph{Upper bound.} 
To complete the proof of \autoref{thm:2}. We show that the algorithm from~\cite{NIPS2018_8174} can be improved. Our new scheduling scheme proceeds in two stages. First, it follows the round-robin algorithm until the consistency constraint is tight. Then, it processes jobs in a greedy order, starting with the job of minimum prediction time. We name the algorithm {\texttt{Two-Stage-Schedule}} and prove the following guarantee: 
\begin{lemma}[Algorithm for non-clairvoyant scheduling]\label{lem:ubs}
For the non-clairvoyant scheduling problem of $2$ jobs,
the algorithm {\texttt{Two-Stage-Schedule}} achieves $(1+\lambda)$-consistency and $(1+1/(1+6\lambda))$-robustness for a $\lambda\in (0,1/3)$.
\end{lemma}
The proof can be found in Appendix~\ref{sec:ubs}. Finally, combining \cref{lem:ubs} and \cref{lem:lbs} proves \autoref{thm:2}.

\section{Conclusion}
In this paper, we give lower bounds for the learning-augmented versions of the ski-rental problem and non-clairvoyant scheduling. In doing so, we show that robustness-consistency trade-offs are deeply intrinsic to the design of online algorithms that are robust in the worst case yet perform well when machine-learned predictions are accurate. 

A broad future direction is to use our techniques to investigate tight robustness-consistency trade-offs for other learning-augmented online algorithms (e.g., online matching or generalizations of ski-rental) following the spate of recent works on this topic.

\section*{Acknowledgements}
We would like to thank Constantinos Daskalakis, Piotr Indyk, and Jelani Nelson for their comments on drafts of this paper. 

\bibliographystyle{alpha}
\bibliography{bib}

\clearpage
\appendix

\section{Proof of~\autoref{thm:rand-low}} \label{apx:A}
\paragraph{LP Construction}
First, consider the following  LP construction for the learning-augmented ski-rental problem:
We use $\rob$ to denote the robustness parameter and $\beta$  the consistency parameter. 
We assume without loss of generality that $\beta < \gamma$; otherwise, the consistency requirement is redundant.
Consider a infinite LP, with variables $\{p_i\}$ indicating the probability of buying at day $i$. 
First, we ask that the $p_i$'s define a probability distribution. That is, $p_i \geq 0$ and 
\begin{align}
    \sum_{i=1}^\infty p_i= 1 \label{eqn: prob} \tag{probability distribution}
\end{align}
Second, to satisfy the consistency constraint, the algorithm must have expected cost within $\beta\cdot \OPT$ when $y = x$. In this case, the ski season ends at $i=y$, so there is no additional cost afterwards.
\begin{align}
 \sum_{i=1}^y (B+i-1)p_i
   +
   y\sum_{i=y+1}^\infty  p_{i}  \leq \beta \min\{ B , y\}.
   \label{eqn: cons} \tag{consistency}
\end{align}
Third, each value of $x$ gives a distinct constraint for robustness, where the left side is the expected cost and the right side is $\gamma \cdot \OPT$.  When $x\leq B$, $\OPT = x$, so we have
\begin{align}
\sum_{i=1}^x (B+i-1)p_i + x\sum_{i=x+1}^\infty p_i \leq \rob x
\quad\forall x \leq B
\end{align}
If $x> B$, then $\OPT = B$. The robustness constraints are  infinitely many, given by
\begin{align}
\sum_{i=1}^x (B+i-1)p_i + x\sum_{i=x+1}^\infty p_i \leq \rob B
\quad\forall x > B
\end{align}

We remark that for each $p_i$, its coefficient is non-decreasing as we go down. We denote the robustness constraint corresponding to $x$ by $\mathcal{C}(x)$ and the entire (infinite) LP by $\mathcal{P}$.

From now on, we focus on the case when $y\geq 2B-1$. In this case, the consistency constraint is
  \begin{align}
   \sum_{i=1}^y (B+i-1)p_i
   + y\sum_{i=y+1}^\infty  p_{i} \leq \beta B. \label{eqn: consa} \tag{consistency'}
\end{align}

\paragraph{Reducing to a Finite LP}
We will show that $\mathcal{P}$ can be reduced to be a finite LP of $B+1$ constraints and $y$ variables. 

\begin{lemma}\label{lem:by}
 If $y \geq B$,  the robustness constraints between $\mathcal{C}(B)$ and $\mathcal{C}(y)$ are redundant.
\end{lemma}
\begin{proof}
Observe that these constraints are dominated by the  constraint \ref{eqn: consa}, since $\beta \leq \gamma$ and their left sides are bounded by the left side of \ref{eqn: consa}.
\end{proof}
    
 We start by identifying redundant variables in $\mathcal{P}$.
 
\begin{lemma}[Redundant variables] \label{lemma:y+1}
  If $y\ge 2B-1$ and $\mathcal{P}$ is feasible, then there exists a feasible solution to $\mathcal{P}$ such that $p_{i}= 0$ for all $i\geq y+1$.
\end{lemma}
\begin{proof}
  Let $p$ be a feasible solution to $\mathcal{P}$.  
First, we argue that there exists a feasible solution $p'$ such that $p'_{y+2} = p'_{y+3}=\cdots =0$. 
To eliminate $p_{y+2}$, consider $p'$ where $p'_{y+2} = 0$, $p'_{y+1} =p_{y+1} + p_{y+2}$, and $p'_i = p_i$ for $i\notin \{y+1,y+2\}$. 
 Clearly, $p'$ still defines a probability distribution.  We now check $p'$ is feasible.
\begin{enumerate}[(i)]
    \item 
Since $y\geq B$, the consistency constraint  is satisfied.  
\item The robustness constraints from $\mathcal{C}(1)$ to $\mathcal{C}(B-1)$ are satisfied, by the coefficients in these constraints.

\item By~\hyperref[lem:by]{Lemma}~\ref{lem:by}, we focus on the robustness constraints from $\mathcal{C}(y+1)$. First, the coefficient of $p_{y+2}$ is greater than the coefficient of $p_{y+1}$ in the constraint $\mathcal{C}(i)$. for all $i\geq y+2$. Hence, $p'$ satisfies these constraints. 
Then, note that the constraint  $\mathcal{C}(y+1)$ is dominated by the constraint $\mathcal{C}(y+2)$, and thus $p'$ satisfies it.
\end{enumerate}
Applying this argument iteratively, we can eliminate all variables $p_i$ for $i \geq y+2$. Now it is easy to observe that constraints $\mathcal{C}(j)$ for $j\geq y+2$ are redundant.

Finally, to eliminate $p_{y+1}$, consider $p'_{y+1} = 0, p'_{y-B+1} =  p_{y-B+1} + p_{y+1}$, and $p'_i = p_i$ for $i\notin \{y-B+1,y+1\}$.  Observe that since $y\geq 2B-1$, 
\begin{enumerate}[(i)]
    \item    in the \ref{eqn: consa} constraint and constraint $\mathcal{C}(y+1)$,  the coefficient of $p_{y+1}$ and  $p_{y-B+1}$  are both $y$;
    \item in the constraints between $\mathcal{C}(1)$ and $\mathcal{C}(B-1)$,  the coefficient of $p_{y+1}$ and $p_{y-B+1}$ are both $B$.
\end{enumerate}
 It follows that all constraints are satisfied by $p'$.
\end{proof}
\begin{corollary}[Redundant constraints]
If $y\geq 2B-1$,  all robustness constraints in $\mathcal{P}$ are redundant except those between $\mathcal{C}(1)$ and $\mathcal{C}(B-1)$.
\end{corollary}
\begin{proof}
  This follows directly from the definition of the robustness constraints and Lemma~\ref{lemma:y+1}.
\end{proof}

\begin{lemma}[Tight constraints]\label{lemma:tight}
  If $\mathcal{P}$ is feasible, there exists a solution such that for all $1 < k < y$, if $p_k > 0$, then the preceding robustness constraints $\mathcal{C}(k')$ for $1\le k' < k$ are all tight.
\end{lemma}
\begin{proof}
  Using the \ref{eqn: prob} constraint, we can rewrite each robustness constraint $\mathcal C(i)$ as
  \begin{equation}
  \label{eqn:reform}
      (B-i)p_1 + (B- i+1)p_2 + \cdots + (B-1)p_i \le\rob\min\{B,i\} - i.
  \end{equation}
 Let $p$ be a feasible solution.  First, we claim that  when shifting probability mass from $p_{k}$ to $p_{k-1}$, the slack for all robustness constraints is non-decreasing, except for $\mathcal C(k-1)$.  Note that 
  \begin{enumerate}[(i)]
      \item for $k'\ge k$, since the coefficient of $p_{k-1}$ in $\mathcal C(k')$ is less than that of $p_k$,  we  strictly increase the slack; and
      \item 
       for $k' < k-1$, $\mathcal C(k')$ has no dependence on either $p_k$ or $p_{k-1}$, so the slack remains unchanged.
  \end{enumerate}
Second, we claim that if $\mathcal C(k-1)$ has non-zero slack, then $p_k > 0$ or $\mathcal C(k)$ has non-zero slack. Indeed, if $p_k = 0$, then constraint $\mathcal C(k-1)$ is stronger than constraint $\mathcal C(k)$. Since constraint $\mathcal C(k-1)$ has non-zero slack, constraint $\mathcal C(k)$ must also have non-zero slack.
  
  Let $s = (s_1,s_2,\ldots,s_{y-1})$ be the vector of slacks for the robustness constraints $\mathcal C(i)$. The above two claims together show that if a feasible solution is such that $p_k > 0$, but one of the preceding robustness constraints is not tight, then we can shift some probability mass so that the robustness slack vector becomes lexicographically smaller. (Here, we consider the lexicographic ordering on $\mathbb{R}^{y-1}$.) Equivalently, if the robustness slack vector is lexicographically minimal, then $p_k > 0$ implies all of the preceding robustness constraints are tight. It thus suffices to show that a lexicographically minimal robustness slack vector exists.
  
  Let $S$ be the set of robustness slack vectors that correspond to feasible solutions. Observe that $S$ is compact. The existence of a lexicographically minimal element in $S$ then follows because any compact subset of $\mathbb{R}^{y-1}$ contains a lexicographically minimal element. To see this last point, note that compactness implies there exists a element with minimum first coordinate. Now, restrict our compact set to this minimum first coordinate and repeat this argument for the second coordinate, and so on.
\end{proof}

Finally, to prove our main theorem, we need the following technical lemma.
  \begin{lemma}\label{lemma:tech}
  For all $B > 1$ and $x\in [0, 1]$, the following inequality holds:
  \[ \frac{1}{B}x-\left(1+\frac{1}{B-1}\right)^{-1}\left(\left(1+\frac{1}{B-1}\right)^x-1\right)\ge 0. \]
  \end{lemma}
  \begin{proof}
    The stated inequality is equivalent to
    \[ \left(1 + \frac{1}{B-1}\right)\frac{1}{B}x + 1\ge \left(1 + \frac{1}{B-1}\right)^x. \]
    Observe that the left-hand side is linear and that the right-hand side is convex. Since the two sides are equal at $x = 0$ and $x = 1$, the desired inequality is true on the interval $[0, 1]$ by Jensen's inequality.
  \end{proof}
  Now we are ready to present the proof of the randomized lower bound.
\begin{proof}[Proof of~\autoref{thm:rand-low}]
  Fix a cost $B$. Given any value of robustness $\rob > 1$, we give a lower bound for the consistency $\cons$ in the ``hard'' case $y = 2B-1$. Assume for now that the LP with parameters $\rob$ and $\cons$ is feasible. We will derive constraints on $\cons$ in terms of $\rob$.
  
  By \Cref{lemma:y+1}, we know that there exists a feasible solution $p$ with $p_{y+1} = 0$. By \Cref{lemma:tight}, there exists a $k\le y$ such that $p_i > 0$ for $1\le i\le k$ and $p_i = 0$ for all $i > k$. Moreover, the first $k - 1$ constraints of the LP all have $0$ slack. In fact, $k\le B$ always: By our feasibility assumption, $\rob$ must be at least the optimal competitive ratio $c^* = e/(e-1)$ in the classic setting, since that setting has no consistency constraint. In the classic setting, we can achieve the optimal competitiveness with $k = B$ and have the first $k$ robustness constraints be tight \cite{DBLP:journals/algorithmica/KarlinMMO94}. 
  Thus, with  $\rob \geq c^*$, the robustness constraints are relaxed, so  we must also have $k\le B$. With these observations and the constraint $p_1 + p_2 + \cdots + p_k = 1$, we can determine the value of $k$ and the probabilities $p_1,\ldots,p_k$.
  
  It is not difficult to see via induction on $k$ (\eg, using the reformulation of \Cref{eqn:reform}) that for the first $k - 1$ constraints to each have $0$ slack, we must have
  \[ p_i = \frac{\rob - 1}{B - 1}\left(1 + \frac{1}{B-1}\right)^{i-1} \]
  for each $i$ between $1$ and $k-1$. From $p_1 + p_2 + \cdots + p_k = 1$, it follows that $k$ must be the smallest integer such that
  \[ \sum_{i=1}^k \frac{\rob - 1}{B - 1}\left(1 + \frac{1}{B-1}\right)^{i-1} = (\rob - 1)\left(\left(1 + \frac{1}{B-1}\right)^k - 1\right)\ge 1. \]
  Rearranging the inequality now gives us
  \[ k = \left\lceil\frac{\log\left(1 + \frac{1}{\rob - 1}\right)}{\log\left(1 + \frac{1}{B - 1}\right)} \right\rceil. \]
  This choice of $k$, our definition of $p_i$ for $1\le i\le k-1$, and the constraint $\sum_{i=1}^k p_i = 1$ fully determine a feasible solution.
  
  By our assumption that the LP is feasible, this setting of $p_i$ values must also satisfy the consistency constraint. That is,
  \begin{equation}\label{equation:cons}
  \sum_{i=1}^k (B + i - 1)p_i = B + \sum_{i=1}^k (i - 1)p_i\le \cons B.
  \end{equation}
  The remainder of this proof consists of computing the left-hand side of the above for our feasible solution explicitly to obtain a lower bound for the consistency $\cons$.
  
  Applying our explicit formulas for $p_i$ for $1\le i\le k-1$ and the fact $p_k = 1 - \sum_{i=1}^{k-1} p_i$, we compute the sum
  \begin{align*}
      \sum_{i=1}^k (i - 1)p_i
      &= \left(\sum_{i=1}^{k-1} (i - 1)p_i\right) + (k-1)\left(1 - \sum_{i=1}^{k-1} p_i\right)  \\
      &= (\rob - 1)\left( B - (B - k + 1)\left(1 + \frac{1}{B-1}\right)^{k-1}\right) \\ & \hspace{9em}+ (k-1)\left(1 - (\rob - 1)\left(\left(1 + \frac{1}{B-1}\right)^{k-1} - 1\right) \right) \\
      &= (k-1)\rob + (\rob - 1)B\left(1 - \left(1 + \frac{1}{B-1}\right)^{k-1}\right).
  \end{align*}
  It follows from \Cref{equation:cons} that $\beta$ is lower bounded by
  \begin{align*}
    \cons
    &\ge 1 + \frac{1}{B} \sum_{i=1}^k (i - 1)p_i \\
    &= 1 + \frac{(k-1)\rob}{B} + (\rob - 1)\left(1 - \left(1 + \frac{1}{B-1}\right)^{k-1}\right).
  \end{align*}
  Now, define
  \[ \Dk\coloneqq k - \frac{\log\left(1 + \frac{1}{\rob - 1}\right)}{\log\left(1 + \frac{1}{B-1}\right)}. \]
  After some further computation, we obtain
  \begin{align*}
    \cons
    &\ge \frac{k\rob}{B} - \rob\left(1 - \frac{1}{B}\right)\left(\left(1 + \frac{1}{B-1}\right)^{\Dk} - 1\right) \\
    &= \frac{\rob}{B}\frac{\log\left(1 + \frac{1}{\rob - 1}\right)}{\log\left(1 + \frac{1}{B-1}\right)} + \rob\left(\frac{\Dk}{B} - \left(1 + \frac{1}{B-1}\right)^{-1}\left(\left(1 + \frac{1}{B-1}\right)^{\Dk} - 1\right)\right).
  \end{align*}
  Lemma~\ref{lemma:tech} lets us bound the terms involving $\Dk$ from below, and it follows that
  \[ \beta
    \ge
    \frac{\rob}{B}\frac{\log\left(1 + \frac{1}{\rob - 1}\right)}{\log\left(1 + \frac{1}{B-1}\right)}. \]
  To finish our proof of the lower bound, notice that $B\log(1 + 1 / (B-1))\to 1$ from below as $B\to\infty$. Hence, the lower bound on $\cons$ approaches $\gamma\log(1 + 1 / (\gamma - 1))$ as $B\to\infty$.
\end{proof}

\section{Proof of \autoref{lem:ubs}}\label{sec:ubs}
Now we present our algorithmic result. Although our analysis deals with the case of $2$ jobs, 
it is  convenient to describe the algorithm in the general case of $n$ jobs.  
The algorithm starts by running round robin for a while, then switches to a greedy strategy of processing jobs in the increasing order of the predicted times. 
If at any point we know $x_i \neq y_i$ for any job $i$, 
we switch to round robin forever. We use $\textsf{OPT}_y = \sum_{i} iy_i$ to denote the $\textsf{OPT}$ under perfect predictions.

\begin{algo}
  \textul{$\texttt{Two-Stage-Schedule}$}$(y_1,y_2,\cdots,y_n)$:\+\\
    At any point, if a job finishes with processing time less or more than its prediction, \+\\
    round robin forever.\-\\
    \textit{Stage} $1$: Round robin for at most $\lambda n \cdot \textsf{OPT}_y /\binom{n}{2}$ units of time.\\
    \textit{Stage} $2$: Process jobs in predicted order\\ 
    \quad \quad \quad\,\,\,(staring from the unfinished job with the least predicted time).
\end{algo}
The intuition behind the algorithm is simple. On one hand,
to ensure robustness, the algorithm switches to round robin when any misprediction is noticed. 
On the other hand, we ask the algorithm to be   $(1+\lambda)$-consistent.   
Suppose $y_1<y_2< \cdots < y_n$. 
If the predictions are perfect, then we expect that a consistent algorithm would produce a schedule that finishes the jobs in the correct order, \ie, job $1$ finishes   first, job $2$    second, and so on. 
In this case, the consistency requirement reduces to
\begin{equation}\label{eq:csr}
    \sum_{i>j} d(i,j) \leq \lambda\,     \textsf{OPT}_y,
\end{equation}
where  and $d(i,j)$ denotes the amount job $i$ delays job $j$ in this scenario.  
Observe that when no job is completed, round robin increases each term in the summation at the same rate of $1/n$.
Thus, stage 1 of the algorithm would make the inequality~\eqref{eq:csr} tight.
Then as we can no longer disobey the predictions in the ideal scenario, we switch to the greedy strategy in the second stage. Next, we analyze the performance of the algorithm in the case of two jobs.

We now prove \autoref{lem:ubs}

\begin{proof}[Proof of \autoref{lem:ubs}]
Let $t= 2y_1 + y_2$. To show consistency, assume $x_1= y_1, x_2= y_2$, so $\textsf{OPT} = t$. In stage $1$, the algorithm runs round robin for $2\lambda t$ units of time. Observe that job $2$ cannot finish before job $1$ in this stage: since $\lambda<1/3$, job $2$ can receive at most $(2y_1+y_2)/3 <y_2$ units of processing time. Consider two cases.
\begin{enumerate}[(i)]
    \item Suppose job $1$ finishes in stage $1$. Then since two jobs share the same rate, 
    \begin{equation}\label{eq:y1s}
        y_1 \leq \lambda t. 
    \end{equation}
    Moreover, in this case. the algorithm runs round robin for $2y_1$ time  and finishes job $2$ in $y_2-y_1$ time.  Thus, $\textsf{ALG} = 3y_1+y_2$, and $\textsf{OPT} = t$. By~\eqref{eq:y1s}, we have $\textsf{ALG} \leq (1+\lambda)\,\textsf{OPT}$.
    \item Suppose job $1$ does not finish in stage $1$. Then both jobs have been processed for $\lambda t $ units of time at the beginning of stage $2$. In stage $2$, the algorithm prioritizes job $1$. Thus,
    \begin{equation}
        \textsf{ALG} = 4\lambda t+ 2(y_1-\lambda t) + (y_2-\lambda t) = (1+\lambda)\,\textsf{OPT}
    \end{equation}
\end{enumerate}
To show robustness, we consider mispredictions, and suppose without loss of generality $y_1 =1$.
Throughout, we let $\epsilon$ to denote an infinitesimal  quantity.
Notice that if any misprediction is found or job $1$ is finished in stage $1$, the algorithm is equivalent of round robin and, therefore, achieves $4/3$ competitive ratio that is better than $1+1/(1+6\lambda)$ for any $\lambda \in (0,1/3)$, so we are done. 
We do a case-by-case analysis, assuming  in stage $1$  no misprediction is detected and both jobs  are finished in stage $2$.
Notice that under the assumptions, $x_1,x_2 \geq \lambda t$, so $\textsf{OPT} \geq 3\lambda t$.
\begin{enumerate}[(i)] 
    \item Suppose  job $1$ finishes no later than its prediction ($x_1 \leq 1$). We have $\textsf{ALG} = \lambda t + 2x_1 +x_2$. \label{it:i}
    \begin{enumerate}[(a)]
        \item If $x_1 < x_2$, then $\textsf{OPT} = 2x_1 + x_2$. Since $\lambda t \leq \textsf{OPT} /3$, we have $\textsf{ALG}/\textsf{OPT} \leq 4/3$.  \label{it:ia} 
        \item If $x_1 \geq  x_2$, then $\textsf{OPT} = 2x_2 + x_1$. Observe that  setting $x_1=y_1= y_2 =1, x_2 = \lambda t + \epsilon$ maximizes the competitive ratio, and this yields a ratio of $1+1/(1+6\lambda)$.   \label{it:ib} 
    \end{enumerate}
    \item Suppose  job $1$ finishes later than its prediction ($x_1 > 1      $). In this case, the stage $2$ starts off by processing job $1$ for $y_1 - \lambda t$ unit of time then switching to round robin.
    \begin{enumerate}[(a)]
        \item If job $1$ finishes no later than job $2$, then we calculate that
        $
        \textsf{ALG} =
        \lambda t +3 x_1 +x_2 -1. 
        $
        If $x_1 < x_2$,   then $\textsf{OPT} = 2x_1 + x_2$, the competitive ratio is at most $4/3$, where the worst case is achieved at $x_1 = 1+\epsilon$ and we use $\lambda t \leq \textsf{OPT}/ 3$. 
        If $x_1 \geq x_2$, then $\textsf{OPT} = 2x_2 + x_1$. The competitive ratio is bounded by $1+1/(1+6\lambda)$, where the worst case is achieved when  $x_1 = 1+\epsilon, x_2 = \lambda t + 2\epsilon, y_2= 1$.
        \item If job $1$ finishes later than  job $2$,  then $\textsf{ALG} = 1 + x_1 + 3x_2 -\lambda t$. Observe that in this case, it is impossible that $x_2 > x_1$, since job $1$ receives more processing than job $2$ throughout the schedule. Assume $x_2 \leq x_1$; then the competitive ratio is bounded by $1+1/(1+6\lambda)$ with the worst case being $x_2=\lambda t + \epsilon,x_1=1$.
        \qedhere
    \end{enumerate}
\end{enumerate}
 \end{proof}
\end{document}